%% file: arxiv.tex
\documentclass[accepted]{uai2021} 

\usepackage[american]{babel}

\usepackage{natbib} 
\usepackage{mathtools} 
\usepackage{booktabs} 
\usepackage{tikz} 

\usepackage{accents}
\usepackage{amsmath}
\usepackage{amsfonts}
\usepackage{amssymb}
\usepackage{amsthm}
\usepackage{appendix} 
\usepackage{graphicx}
\usepackage{verbatim}
\usepackage{algorithm}
\usepackage{algorithmic}
\usepackage{caption}
\usepackage{accents}
\usepackage{subcaption}
\usepackage[T1]{fontenc}
\usepackage{xfrac}

\DeclareMathOperator*{\argmax}{arg\,max}
\newcommand{\ubar}[1]{\underaccent{\bar}{#1}}
\newcommand{\utilde}[1]{\underaccent{\tilde}{#1}}

\newtheorem{theorem}{Theorem}[section]

\newtheorem{lemma}{Lemma}[section]
\newtheorem{definition}{Definition}[section]
\newtheorem{remark}{Remark}[section]

\title{Strategically Efficient Exploration in Competitive Multi-agent \\ Reinforcement Learning}

\author[1]{\href{mailto:R.T.Loftin@tudelft.nl}{Robert~Loftin\thanks{This work done while at Microsoft Research Cambridge}}{}} 
\author[3]{\href{mailto:Aadirupa.Saha@microsoft.com}{Aadirupa~Saha}{}}
\author[2]{\href{mailto:Sam.Devlin@microsoft.com}{Sam~Devlin}{}}
\author[2]{\href{mailto:Katja.Hofmann@microsoft.com}{Katja~Hofmann}{}}
\affil[1]{
    TU Delft \\
    Delft, Netherlands
}

\affil[2]{
    Microsoft Research Cambridge \\
    Cambridge, UK
}

\affil[3]{
    Microsoft Research NYC \\
    New York, USA
}

\begin{document}
\maketitle

\begin{abstract}

    High sample complexity remains a barrier to the application of reinforcement learning (RL), particularly in multi-agent systems.  A large body of work has demonstrated that exploration mechanisms based on the principle of \emph{optimism under uncertainty} can significantly improve the sample efficiency of RL in single agent tasks.  This work seeks to understand the role of optimistic exploration in non-cooperative multi-agent settings.  We will show that, in zero-sum games, optimistic exploration can cause the learner to waste time sampling parts of the state space that are irrelevant to strategic play, as they can only be reached through cooperation between both players.  To address this issue, we introduce a formal notion of \emph{strategically efficient} exploration in Markov games, and use this to develop two strategically efficient learning algorithms for finite Markov games. We demonstrate that these methods can be significantly more sample efficient than their optimistic counterparts.

\end{abstract}

\section{Introduction}
    
    Despite its success in recent years, the applicability of reinforcement learning is still limited by the enormous amounts of training data required to solve complex tasks, particularly when those tasks involve multiple agents \citep{vinyals2019alphastar,berner2019dota}.  For single-agent problems it has been shown that sample efficiency can be significantly improved with the use of more sophisticated exploration mechanisms that take into account the learner's own uncertainty about the learning task~\citep{pathak2017icm,burda2018rnd}. Extending these approaches to multi-agent settings, however, remains an open challenge. 
    
    In this work, we focus on efficient exploration for reinforcement learning in competitive multi-agent settings.  In recent related work, \cite{bai2020ulcb} have presented algorithms for self-play in finite Markov games with sample complexity bounds that are polynomial in the size of the state and action spaces. These methods are based on the principle of \emph{optimism under uncertainty}, in which each agent acts greedily w.r.t. a statistically plausible model of the learning task that maximizes the agent's expected return.  In two-player games, this optimism encourages the players to cooperate to reach states that have not previously been observed (driven by the assumption that both players can receive large positive returns from such unknown states).  In zero-sum games, however, such cooperative behavior would never be observed between rational opponents.
    
    In this paper, we show that such cooperative exploration is \emph{strategically inefficient}, and may cause the learner to waste time exploring parts of the state space that provide no additional information about the Nash equilibria of the game.  The key question for this work is how a reinforcement learning algorithm can recognize and avoid such strategically irrelevant parts of the state space, while still ensuring that an approximate solution to the game will be found.  To address this question, we propose two reinforcement learning algorithms, 
    \emph{Strategic ULCB} and \emph{Strategic Nash-Q}, which are strategically efficient in a suitably well-defined sense.  As with the optimistic algorithms of \cite{bai2020ulcb}, these algorithms select exploration policies optimistically w.r.t. a set of statistically plausible games.  However, unlike prior work, in our approach each player chooses an optimistic best-response against the strongest known adversary strategy (rather than its opponents' exploration strategy).  
    
    In Section~\ref{def:strategic_efficiency} we will prove that that Strategic ULCB is both strategically efficient and sample efficient in the traditional sense, while in Section~\ref{experiments} we will show that Strategic ULCB and Nash-Q significantly outperform their existing, optimistic counterparts.  Our key conclusion is that the direct extension of optimistic exploration to multi-agent RL in competitive settings can be highly inefficient, and that by leveraging the adversarial nature of zero-sum games, it is possible to dramatically improve sample efficiency through the use of strategically efficient exploration mechanisms.

\section{Preliminaries}
    \label{preliminaries}

    This work focuses on the role of exploration in finite, two-player zero-sum Markov games~\citep{littman1994markov}.  We define such a Markov games as a tuple $G = \{S, A, B, P, R, H \}$.  Here $S$ is a finite state space, and we let $h \in [1,H]$ be the steps since the start of the current episode.  $A_{h,s}$ and $B_{h,s}$ are state and step-dependent action spaces $A_{h,s}$ for the min and max-players respectively, $P_h: S \times A \times B \mapsto \mathbb{P}(S)$ is the step-dependent transition distribution, $R_h : S \times A \times B \mapsto [0, 1]$ is the step-dependent reward function for the max player, and $H$ is the fixed episode length.  Let $\vert S \vert = \max_{h}\vert S_h \vert$, $\vert A \vert = \max_{s,h}\vert A_{h,s} \vert$ and $\vert B \vert = \max_{s,h}\vert B_{h,s} \vert$. We assume that rewards are deterministic.  For zero-sum games, we need only specify the reward function for the max-player, with the reward for the min-player defined as $-R_h(s, a, b)$.  The restriction to Markov games implies that the state is fully observable to both agents at all times.  We also assume that there is a unique initial state $s_1$. 
    
    Training proceeds episodically for $K$ episodes of length $H$. For the state $s^{k}_h$ encountered at step $h$ of episode $k$, the learner samples actions $a^{k}_h \in A_{h, s^{k}_h}$ and $b^{k}_h \in B_{h, s^{k}_h}$ from the joint exploration policy $\pi^{k}_h(s, a, b)$.  After taking joint action $(a^{k}_h, b^{k}_h)$, the learner observes reward $r^{k}_h = R_h(s^{k}_h, a^{k}_h, b^{k}_h)$, and state $s^{k}_{h+1} \sim P_h(s^{k}_h, a^{k}_h, b^{k}_h)$ if $h < H$.  We assume here that the exploration policy $\pi^{k}_h : S \mapsto \mathcal{P}(A_{s^{k}_h} \times B_{s^{k}_h})$ is computed in advance for all $s$ and $h$, and fixed throughout episode $k$. When the exploration policy can be factored into separate policies for the max and min-players, we denote these as $\mu^k$ and $\nu^k$ respectively, with $\pi^{k}_h(s,a,b) = \mu^{k}_h(s,a)\nu^{k}_{h}(s,b)$.
    
    For any pair of policies $\mu$,$\nu$, we define $V^{\mu,\nu}_h(s)$ to be the expected return of the max player from state $s$ at step $h$ as:
    \begin{equation}
        \label{eqn:value}
        V^{\mu,\nu}_h(s) = \text{E}\left[ \sum_{i=h}^{H} r_i(s_i, a_i, b_i) \vert \mu,\nu, s_h = s\right]
    \end{equation}
    When training in self-play, we have no way of knowing what adversary the policies we learn will eventually need to play against.  We therefore evaluate our learned policies $\mu$ and $\nu$ in terms of their worst-case return against any adversary policy, which we define as their \emph{exploitability}
    \begin{equation}
        \text{expl}(\mu) =  - \inf_{\nu'} V^{\mu, \nu'}_1(s_1), \hspace{.3cm} \text{expl}(\nu) = \sup_{\mu'}V^{\mu', \nu}_1(s_1)
    \end{equation}
    For a pair of policies $\mu$ and $\nu$, the total exploitability is equal to the $\text{NashConv}$ loss~\citep{johanson2011nashconv,lanctot2017psro}, defined as
    \begin{equation}
        \label{eqn:nash_conv}
        \text{NashConv}(\mu, \nu) = \sup_{\mu'}V^{\mu', \nu}_1(s_1) - \inf_{\nu'} V^{\mu, \nu'}_1(s_1),
    \end{equation}
    If $\mu,\nu$ constitute a Nash equilibrium of the game, then $\text{NashConv}(\mu,\nu) = 0$, and if $\text{NashConv}(\mu,\nu) \leq \epsilon$, then $\mu$ and $\nu$ will constitute an $\epsilon$-Nash equilibrium of the game.

\section{Related Work}
    \label{related}
    
    While we focus on exploration in finite Markov games, this work is motivated by the goal of extending exploration approaches that have proven effective in single-agent deep reinforcement learning to the competitive multi-agent setting.  Many successful approaches guide exploration by providing an additional \emph{intrinsic} reward signal that is larger for states and actions for which the learner is less certain, often referred to as \emph{curiosity}~\citep{burda2018curiosity}.  These include the Intrinsic Curiosity Module~\citep{pathak2017icm}, which uses the error of a supervised transition model to estimate uncertainty, and Random Network Distillation~\citep{burda2018rnd}, which uses the error between a fixed, randomly initialized network and a prediction network trained on the states the learner has observed so far. Such intrinsic rewards have also proven effective in cooperative multi-agent RL, where all agents aim to maximize a common reward function~\citep{iqbal2019coordinated,bohmer2019intrinsic}.
    
    While not subject to the same theoretical guarantees, the use of uncertainty-based intrinsic rewards in deep RL can be motivated by work on finite MDPs, where a number of algorithms based on the principle of optimism under uncertainty have been shown to have good worst-case sample complexity~\citep{jaksch2010ucrl, strehl2008mbie, jin2018qlearning}.  In particular, \cite{strehl2008mbie} and \cite{jin2018qlearning} describe algorithms which incorporate optimism through a count-based exploration bonus of the form $\beta /\sqrt{N(s, a)}$, where $N(s,a)$ is the number of times the state $s$ and action $a$ have been observed previously.  We refer to these algorithms as ``optimistic'' because they select the action that maximizes an upper confidence bound on the expected in the current state, where the upper bound is taken over some set of statistically plausible MDPs.

    Recent results have shown that the use of upper confidence bounds can be extended to self-play in two-player zero-sum Markov games.  \cite{bai2020ulcb} present a model-based self-play algorithm, VI-ULCB (which we will refer to as \emph{Optimistic} ULCB in later sections) that finds an $\epsilon$-equilibrium with at most $O(H^4 \vert S \vert^2 \vert A\vert \vert B\vert /\epsilon^2)$ samples.  VI-ULCB drives exploration by solving for the Nash equilibrium of an optimistic, \emph{general-sum} corresponding to upper and lower confidence bounds on the max-player returns.  \cite{bai2020nash_q} build on this work, presenting a model-free self-play algorithms, Optimistic Nash-Q which finds an $\epsilon$-equilibria in at most $O(H^5 \vert S\vert \vert A\vert \vert B\vert /\epsilon^2)$ samples.
    
    While these bounds are near-optimal in the worst case, they say little about the practical efficiency of these algorithms, or the approaches to exploration that they embody.  These results do not rule out the possibility that the learner will need to explore the entire state-action space, even when this is unnecessary for the identification of an $\epsilon$-equilibrium of the game. More specifically, in Section~\ref{strategic_exploration} we will show that VI-ULCB can select pairs of output policies such that neither policy can plausibly be an equilibrium of the game.  In Section~\ref{experiments}, we will empirically compare Optimistic Nash-Q and VI-ULCB, against two novel algorithms that avoid selecting such implausible policies.  Through these comparisons we will demonstrate that Optimistic Nash-Q and VI-ULCB can suffer from unnecessarily high sample complexity in environments where large parts of the state space are irrelevant the equilibrium solution.

    Finally, we note a connection between the concept of strategically efficient exploration and the Alpha-Beta pruning algorithm from game-tree search~\citep{pearl1980alpha_beta}.  While Alpha-Beta pruning is limited to deterministic, turn-based games with known transition dynamics (and so is not applicable in most RL settings), it nonetheless exploits the adversarial nature of zero-sum games in much the same way that the algorithms developed in this work will.  Like Strategic ULCB and Strategic Nash-Q, Alpha-Beta pruning bounds the value of a state in terms of the strongest adversary strategy it has identified so far, and will not explore states that it knows cannot occur under a minimax optimal strategy for the root player.  Unlike Strategic ULCB and Nash-Q however, Alpha-Beta pruning is not optimistic, and will continue evaluating a set of strategies (corresponding to the current sub-game) even when there exist potentially superior alternatives.

\section{Strategic Exploration}
    \label{strategic_exploration}
    
    \begin{figure}[t]
      \centering
      \includegraphics[width=1.0\columnwidth]{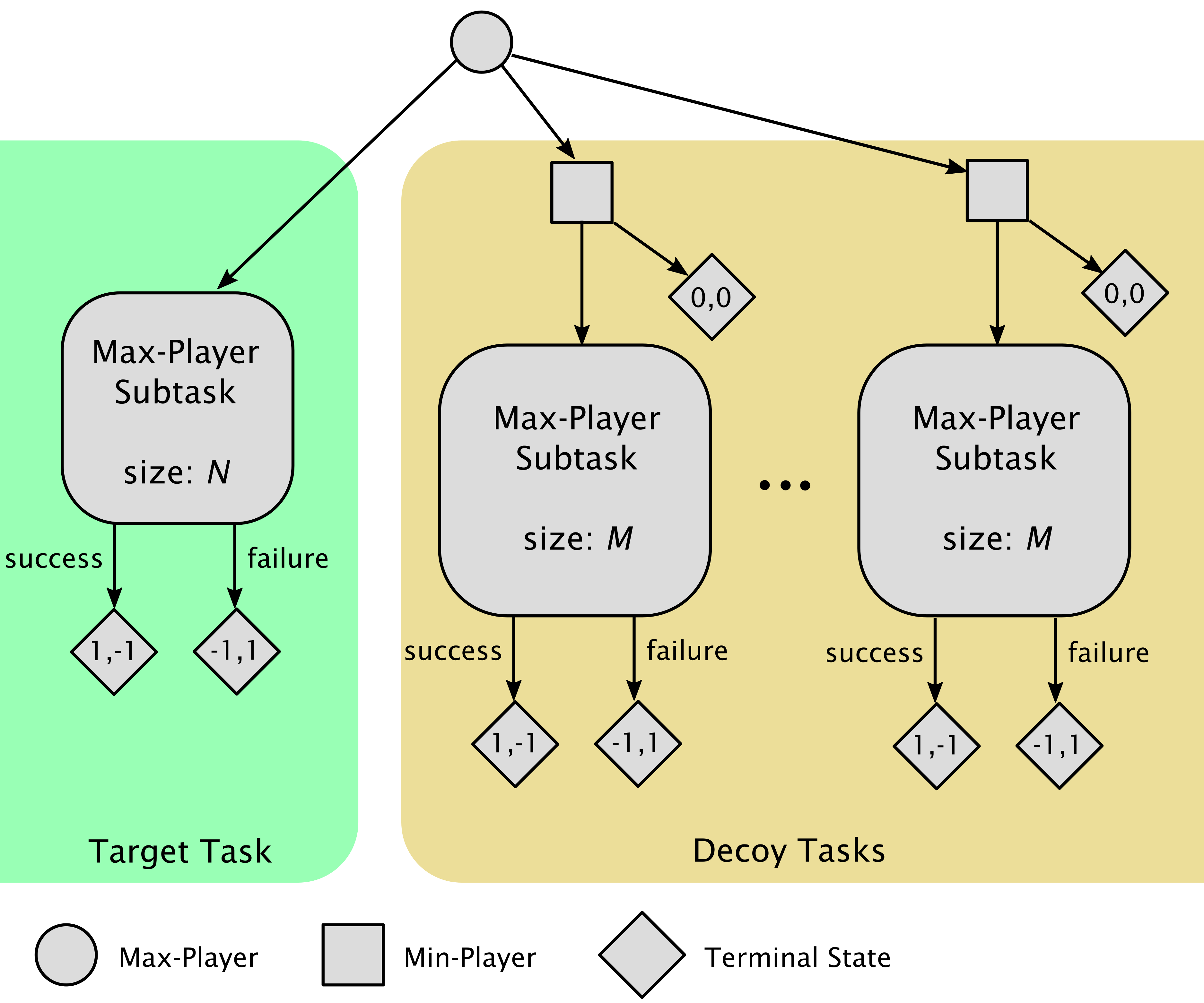}
      \caption{The generic decoy task game.  We can choose any single-player game to define the target task and decoy tasks, as long as the max-player can always succeed in these sub-games with the right policy.  The $(-1,1)$ payoffs correspond to a max-player loss, while the $(0,0)$ is a tie.}
      \label{fig:decoy_game}
    \end{figure}
    
    The worst case sample complexity of optimistic algorithms, such as Optimistic ULCB (Algorithm~\ref{alg:ulcb}) and Optimistic Nash-Q (Algorithm A.1), correspond to the complexity of learning a \emph{complete} model of the game (even for model free algorithms).  In some cases such complete exploration will be necessary, for example, when the game is effectively a single-agent MDP for the max-player, and every possible outcome must be known to ensure the max-player's policy is optimal.  For truly competitive games, however, learning a complete model will often be unnecessary to find a Nash equilibrium.  When this is the case, an optimistic algorithm may waste time exploring parts of the state space that yield no useful information about the solution to the game.
    
    We can illustrate this issue with the abstract, turn-based game shown in Figure~\ref{fig:decoy_game}.  This game, which we will refer to as the \emph{decoy task game}, is composed of a set of single-player sub-tasks in which only the max-player takes actions.  In Section~\ref{experiments} we will show experimental results in this game for a specific choice of sub-task, but for now it is sufficient to assume that each task is ``complex'' in the sense that a learner will have to attempt the task many times before a successful policy is found.  At each episode, the max player first chooses which sub-task they wish to explore.  What is important here is that, for all but one sub-task (the \emph{target task}), the min player has the option to end the game immediately, resulting in a tie, or allow the max-player to attempt the sub-task.  As a result, the learner gains nothing by solving these alternative, or \emph{decoy tasks}, as with or without a solution the best the learner can hope for is a tie if it chooses one of these sub-tasks during evaluation.
    
    In spite of this, an optimistic algorithm may attempt to solve each of the decoy tasks, because, until a sub-task is solved, it will assume that it is possible for both the min and max-players to simultaneously receive a payoff of 1 when that task is complete.  Under an optimistic exploration rule, the learner assumes that the min-player will allow the max-player to complete each of the decoy tasks, when in the underlying game the only reason this would happen is if the max-player gets a payoff $\leq 0$ for completing the task (otherwise the min-player with terminate the game early).  If there are many decoy tasks, optimistic exploration may be highly inefficient.  In this section, we will describe algorithms which, while still sufficiently optimistic to ensure convergence to a solution, will be robust to the existence of such \emph{strategically irrelevant} sub-tasks.  
    
    \subsection{Strategic Efficiency}
    
    To develop strategically efficient algorithms, we will first need to formalize our intuitive notion of strategic efficiency.  Here we define strategic efficiency in terms of the marginal exploration strategies $\mu^{k}$ and $\nu^{k}$ the learner follows during training.  Loosely speaking, a strategically efficient learning algorithm should not consider strategies that it \emph{believes} do not correspond to equilibria of the game.  
    
    Dependence on the agent's own belief is essential, as any strategy could correspond to an equilibrium if we place no restrictions on the set of possible games.  Therefore, we need a representation of the ``belief state'' of a given learning algorithm.  It will be sufficient to consider a representation that is independent of the learning algorithm itself, that is, one which only depends on the observable interactions between the learner and the environment.  We represent the knowledge state by a sequence of sets $C_{1 \leq k} \subset \mathbb{G}(H, S, A, B)$.  Here, $\mathbb{G}(S, A, B)$ is the set of games on $A$, $B$, and the state space $S \cup \{s^{*}\}$ (where $s^{*}$ is a hypothetical absorbing state) and max-player rewards in $[0,H]$.  The absorbing state and larger reward range will simplify the task of proving that an algorithm is strategically efficient.
    
    Each set is $C_k$ itself a random variable, that is, a function $C_{k}(\mathbb{H}_k)$, where $\mathbb{H}_k$ is the history of states, actions, and rewards up to but not including episode $k$.  For $\sigma \in [0,1]$, we say that $C_{0 \leq k}$ are $\sigma$-confidence sets if, under any learning algorithm run on a game $G$
    \begin{equation}
        \text{Pr}\left\{ \exists k \geq 1: G \notin C_k \right\} \leq \sigma
    \end{equation}
    It is possible to define the confidence sets with respect to some subset of $X \subset \mathbb{G}(S, A, B, H)$, such as the set $\mathbb{D}(S, A, B, H)$ of games with deterministic state transitions, so long as we can be certain \emph{a priori} that $G \in D$.  Our definition of strategic efficiency will be with respect to a given sequence of confidence sets. 
    \begin{definition}[Strategic Efficiency]
        \label{def:strategic_efficiency}
        If $C_{1 \leq k}$ are $\sigma$-confidence sets w.r.t. $X \subseteq \mathbb{G}(S, A, B)$, then an algorithm is strategically efficient w.r.t. $C_{1 \leq k}$ if, for all $k \geq 1$, there exists $\tilde{G} \in C_k$ such that
        \begin{equation}
            \exists \tilde{G} \in C_k, \inf_{\nu}V^{\mu^k, \nu}_{\tilde{G},1} \geq \sup_{\mu}\inf_{\nu}V^{\mu, \nu}_{\tilde{G},1}
        \end{equation}
        and there exists $\utilde{G} \in C_k$ such that
        \begin{equation}
            \exists \utilde{G} \in C_k, \sup_{\nu}V^{\mu, \nu^k}_{\utilde{G},1} \leq \inf_{\nu}\sup_{\mu}V^{\mu, \nu}_{\utilde{G},1}
        \end{equation}
    \end{definition}
    Under this definition, a learning algorithm is strategically efficient if its exploration policies are always a component of a \emph{plausible} Nash equilibrium of the true game $G$.   Note that under the trivial sequence $C_{0 \leq k} = \mathbb{G}(S, A, B)$, any learning algorithm would be efficient.  To address this, we will require that the confidence sets converge when data is generated by the algorithm under consideration, that is,  for any $\epsilon > 0$, $\delta \in (0,1]$, there exist $K, \mu, \nu$ s.t. $\text{NashConv}_G(\mu,\nu) \leq \epsilon$ for all $G \in C_K$ with probability at least $1-\delta$.
    
    \subsection{Non-Strategic Exploration}
    \label{non_strategic}
    
    Before discussing the design of strategically efficient learning algorithms, we first demonstrate how the joint-optimism employed by existing approaches can fail to be strategically efficient.  Specifically, we show that Optimistic ULCB can fail to be strategically efficient w.r.t. its own \emph{implicit} confidence sets.   This is easiest to show this for games with deterministic state transitions (which include matrix games with no transitions).  In such games, Optimistic ULCB can be run with an exploration bonus term of $\beta_t = 0$, with efficient exploration being guaranteed by optimistic initialization.  The models maintained by Optimistic ULCB will be exact for all observed $(h,s,a,b) \in \mathbb{H}_k$, and its natural confidence sets will be the sets $D_k$ of games that are exactly consistent the the rewards and state transitions observed up to episode $k$, that is
    \begin{align}
        \label{eqn:version_space}
        \nonumber D_k & = \{G \in \mathcal{D}(S, A, B, H) \vert P_{h}(s^{k}_h, a^{k}_h, b^{k}_h, s^{k}_{h+1}) = 1 \land\\
        &R_h(s^{k}_h, a^{k}_h, b^{k}_h) = r_h, \forall (s^{k}_h, a^{k}_h, b^{k}_h, r^{k}_h, s^{k}_{h+1}) \in \mathbb{H}_{k}\}.
    \end{align}
    
    \begin{remark}
        \label{rmk:optimistic_ULCB}
        Optimistic ULCB is not guaranteed to be strategically efficient with respect to the confidence sets $D_{1\leq k}$ (Equation~\ref{eqn:version_space}) for deterministic games.
    \end{remark}
    We can demonstrate the strategic inefficiency of Optimistic ULCB using a variation on the classical prisoners dilemma.  Similar to the prisoner's dilemma, each player has the option to either cooperate (c) with the other player, or defect (d).  Unlike the original prisoner's dilemma, however, payoffs in this game are zero-sum
    \begin{center}
        \begin{tabular}[t]{r | c c}
              & c & d \\ \hline
            c & (x,-x) & (-.5,.5)  \\
            d & (.5,-.5) & (0,0)
        \end{tabular}
    \end{center}
    with $x \in [-1,1]$.  Note that regardless of the value of $x$, the only equilibrium for this game is the strategy profile in which both players always defect.  Assume now that we have run Optimistic ULCB for three episodes, selecting joint strategies such that the only unobserved combination remaining is joint cooperation.  To compute the exploration strategy for the next episode, Optimistic ULCB will select a Nash equilibrium of the general-sum game
    \begin{center}
        \begin{tabular}[t]{r | c c}
              & c & d \\ \hline
            c & (1,1) & (-.5,.5)  \\
            d & (.5,-.5) & (0,0)
        \end{tabular}
    \end{center}
    where joint cooperation is assumed to yield a payoff of 1 for both players because the true payoffs have never been observed.  Because we know in advance that the game is zero-sum, however, we know that no matter what the true payoff for joint cooperation is, at least one player will have an incentive to defect.  The row player will defect unless its payoff is greater than or equal to .5, but if this is the case. the column player must have a payoff less that or equal to -.5.  This implies that joint defection is the only plausible equilibrium, and that neither player will cooperate as part of an equilibrium strategy for any plausible game.  Therefore, if Optimistic ULCB chooses joint cooperation as its next strategy, it will fail to satisfy Definition~\ref{def:strategic_efficiency}.
    
    \subsection{Strategic ULCB}
    
    \begin{algorithm}[h]
    \begin{algorithmic}[1]
        \small
        \STATE \textbf{Initialize:} $\forall h \in [H]$, $s \in S_h$, $a \in A_{h,s}$, $b \in B_{h,s}$, $s' \in S_{h+1}$, $N^{1}_h(s,a) \leftarrow 0$, $N^{1}_h(s,a,s') \leftarrow 0$.
        \FOR{episode $k = 1,\ldots,K$}
            \FOR{step $h = H, \ldots, 1$}
                \FOR{$s \in S_h$, $a,b \in A_{h,s} \times B_{h,s}$}
                    \STATE $t \leftarrow N^{k}_h(s,a,b)$
                    \STATE $\bar{Q}^{k}_h(s,a,b) \!\!\leftarrow\!\! \min\{\hat{R}^{k}_h(s,a,b) \!\!+\!\! \hat{P}^{k}_h(s,a,b)^{\top}\bar{V}^{k}_{h+1} \!\!+\!\! \beta_t,H\}$
                    \STATE $\ubar{Q}^{k}_h(s,a, b) \!\!\leftarrow\!\! \max\{\hat{R}^{k}_h(s,a,b) \!\!+\!\! \hat{P}^{k}_h(s,a,b)^{\top}\ubar{V}^{k}_{h+1} \!\!-\!\! \beta_t,0\}$
                \ENDFOR
                \FOR{$s \in S_h$}
                    \IF{Strategic ULCB}
                        \STATE $\mu^{k}_{h}(s),\tilde{\nu}^{k}_{h} \leftarrow \text{Nash}(\bar{Q}^{k}_h(s,\cdot,\cdot),-\bar{Q}^{k}_h(s,\cdot,\cdot))$
                        \STATE $\tilde{\mu}^{k}_{h}(s),\nu^{k}_{h} \leftarrow \text{Nash}(\ubar{Q}^{k}_h(s,\cdot,\cdot),-\ubar{Q}^{k}_h(s,\cdot,\cdot))$
                    \ELSIF{Optimistic ULCB}
                        \STATE $\mu^{k}_{h}(s),\nu^{k}_{h} \leftarrow \text{Nash}(\bar{Q}^{k}_h(s,\cdot,\cdot),-\ubar{Q}^{k}_h(s,\cdot,\cdot))$
                        \STATE $\tilde{\mu}^{k}_{h}(s),\tilde{\nu}^{k}_{h} \leftarrow \mu^{k}_{h}(s),\nu^{k}_{h}$
                    \ENDIF
                    \STATE $\bar{V}^{k}_h(s) \leftarrow \mu^{k}_{h}(s)^{\top}\bar{Q}^{k}_h(s,\cdot,\cdot) \tilde{\nu}^{k}_{h}$
                    \STATE $\ubar{V}^{k}_h(s) \leftarrow \tilde{\mu}^{k}_{h}(s)^{\top}\ubar{Q}^{k}_h(s,\cdot,\cdot) \nu^{k}_{h}$
                \ENDFOR
            \ENDFOR
            \STATE set $s^{k}_1 \leftarrow s_1$
            \FOR{step $h = 1, \ldots ,H$}
                \STATE Take actions $a^{k}_h \sim \mu^{k}_h(s^{k}_h)$ and $b^{k}_h \sim \nu^{k}_h(s^{k}_h)$
                \STATE Observe max-player reward $r^{k}_h$ and next state $s^{k}_{h+1}$
                \STATE $N^{k+1}_h(s^{k}_h, a^{k}_h, b^{k}_h) \leftarrow N^{k}_h(s^{k}_h, a^{k}_h, b^{k}_h) + 1$
                \STATE $N^{k+1}_h(s^{k}_h, a^{k}_h, b^{k}_h, s^{k}_{h+1}) \leftarrow N^{k}_h(s^{k}_h, a^{k}_h, b^{k}_h, s^{k}_{h+1}) + 1$
                \STATE $\hat{P}^{k+1}_h(\cdot \vert s^{k}_h, a^{k}_h, b^{k}_h) \leftarrow \frac{N^{k+1}_h(s_h, a^{k}_h, b^{k}_h, \cdot)}{N^{k+1}_h(s_h, a^{k}_h, b^{k}_h)}$
                \STATE $\hat{R}^{k+1}_h(s_h, a^{k}_h, b^{k}_h) \leftarrow r_h$
            \ENDFOR
        \ENDFOR
    \end{algorithmic}
    \caption{The Strategic (and Optimistic) ULCB algorithms. The function $\text{Nash}(G,G')$ computes a mixed strategy profile  $(\mu,\nu)$ constituting a Nash equilibrium of the two-player game given by the payoff matrices $G$ and $G'$.  Strategic ULCB maintains separate evaluation policies $\tilde{\mu}^{k}$ and $\tilde{\nu}^{k}$, while Optimistic ULCB~\citep{bai2020ulcb} uses the same policies for exploration and evaluation.}
    \label{alg:ulcb}
    \end{algorithm}
    
    We now present a model-based learning algorithm that will be provably strategically efficient in some settings.  As this new algorithm is similar in structure to Optimistic ULCB, we refer to it as Strategic ULCB (Algorithm~\ref{alg:ulcb}).  Strategic ULCB differs from Optimistic ULCB in three key ways.  First, it maintains separate policies $\tilde{\mu}^{k}$ and $\tilde{\nu}^{k}$ for evaluation.  This is necessary because strategically efficient exploration may converge to a solution before the game has been fully explored, such that the optimistic exploration policies may remain exploitable indefinitely.  Second, the max-player exploration policy for each state $s$ is defined as a minimax optimal strategy of the matrix game defined by $\bar{Q}^{k}_h(s,\cdot,\cdot)$ (the min-player exploration policy is computed w.r.t. $\ubar{Q}^{k}_h(s,\cdot,\cdot)$).  This focuses exploration on actions that maximize the return a player can \emph{optimistically guarantee} against an adversary.  Finally, the value function updates are
    \begin{align}
        \bar{V}^{k}_h(s) &= \mu^{k}_{h}(s)^{\top}\bar{Q}^{k}_h(s,\cdot,\cdot) \tilde{\nu}^{k}_{h} \\
        \ubar{V}^{k}_h(s) &= \tilde{\mu}^{k}_{h}(s)^{\top}\ubar{Q}^{k}_h(s,\cdot,\cdot) \nu^{k}_{h}
    \end{align}
    which ensures that $\bar{V}^{k}_h(s)$ reflects the best return the max-player can expect against a true adversary, rather than the min-player's exploration policy. To demonstrate the correctness of Strategic ULCB, we provide a bound on the total $\text{NashConv}$ loss incurred by the evaluation policies $\tilde{\mu}^{k}$ and $\tilde{\nu}^{k}$ over $K$ episodes, which we denote as $\text{Regret}(K)$,
    \begin{equation}
        \text{Regret}(K) = \sum_{k=1}^{K}\left[ \sup_{\mu}V^{\mu,\tilde{\nu}^k}_1(s_1) - \inf_{\nu}V^{\tilde{\mu}^k,\nu}_1(s_1)\right]
    \end{equation}
    We will show that $\text{Regret}(K) \leq O(\sqrt{K})$, such that the average $\text{NashConv}$ loss will decay as $O(1/\sqrt{K})$.
    
    \begin{theorem}
        \label{thm:strategic_ulcb_complexity}
        For any $K \geq 3$ and $\delta \geq 0$, if Strategic ULCB (Algorithm~\ref{alg:ulcb}) is run with $\beta_t$ defined as
        \begin{equation}
            \beta_t = H \sqrt{\frac{2\vert S\vert \ell}{t}}
        \end{equation}
        where $\ell = \ln(KH\vert S\vert\vert A\vert\vert B \vert / \delta)$, then its regret satisfies
        \begin{equation}
            \text{\emph{Regret}}(K) \leq 6\sqrt{2KH^{4}\vert S\vert^{2} \vert A\vert \vert B\vert\ell}
        \end{equation}
        with probability at least $1-\delta$.
    \end{theorem}
    
    The full proof of Theorem~\ref{thm:strategic_ulcb_complexity} can be found in Appendix A, and is similar to the proof for the Optimistic ULCB given by~\cite{bai2020ulcb}.  We can sketch the main ideas of the proof by assuming $H=1$ (so we can ignore the state) and that $\ubar{Q}^{k}(a,b) \leq R(a,b) \leq\bar{Q}^{k}(a,b)$.  Because $\mu^{k} = \max_{a}\bar{Q}^{k}(a,\cdot)\tilde{\nu}^{k}$ and $\nu^{k} = \min_{b}(\tilde{\mu}^{k})^{\top}\ubar{Q}^{k}(\cdot,b)$, we have that $\bar{V}^{k} = \max_{a}\bar{Q}^{k}(a,\cdot)\tilde{\nu}^{k} \geq \max_{a}R(a,\cdot)\tilde{\nu}^{k}$, and $\ubar{V}^{k} = \min_{b}(\tilde{\mu}^{k})^{\top}\ubar{Q}^{k}(\cdot,b) \leq \min_{b}(\tilde{\mu}^{k})^{\top}R(a,\cdot)$.  Therefore, the $\text{NashConv}$ loss of the profile $(\tilde{\mu}^{k},\tilde{\nu}^{k})$ is bounded by $\bar{V}^{k} - \ubar{V}^{k}$.  Note that it is not possible to bound the loss of $(\mu^{k},\nu^{k})$ in the same way, and so the need for separate evaluation policies.  We then show that $\bar{V}^{k}$ and $\ubar{V}^{k}$ converge, by showing that they are bounded by the ``on policy'' confidence bounds $\tilde{V}^{k} = (\mu^{k})^{\top}\bar{Q}^{k}\nu^{k}$ and $\utilde{V}^{k} = (\mu^{k})^{\top}\bar{Q}^{k}\nu^{k}$, which do converge under the joint exploration policy. Note that $\tilde{V}^{k} = (\mu^{k})^{\top}\bar{Q}^{k}\nu^{k} \geq (\mu^{k})^{\top}\bar{Q}^{k}\tilde{\nu}^{k}$ because $\tilde{\nu}^{k}$ is also a best-response to $\mu^{k}$, with the same being true for $\tilde{V}^{k}$. 
    
    We can also show that, for the special case of games with deterministic transitions, Strategic ULCB will be strategically efficient with respect to the confidence sets $D_k$ of games that are exactly consistent the the rewards and state transitions observed up to episode $k$.
    \begin{theorem}
        \label{thm:strategic_ulcb_efficiency}
        Strategic-ULCB will be strategically efficient w.r.t. the confidence sets $D_{1\leq k}$ (Equation~\ref{eqn:version_space}) when run with $\beta_t = 0, \forall t$, on any game with deterministic state transitions.
    \end{theorem}
    The proof of Theorem~\ref{thm:strategic_ulcb_efficiency} can be found in Appendix B. The restriction to deterministic games is necessary, as without it $\bar{Q}^k$ and $\ubar{Q}^k$ may not be exactly realizable for any plausible game, which is essential for the proof.  For stochastic games, the bonus terms $\beta_t$ will be approximations of the true upper and lower bounds over the space of statistically plausible games (to see this, consider the value of $\bar{Q}^{k}_h(s,a,b)$ when $\bar{V}^{k}_{h+1} = 0$).  Therefore, Strategic ULCB is will only be \emph{approximately} strategically efficient in stochastic games.
    
    \subsection{Model-Free Algorithms}
    
    \cite{bai2020nash_q} present Optimistic Nash-Q as model-free counterpart to Optimistic ULCB. Optimistic Nash-Q maintains tabular estimates of the upper and lower bounds $\bar{Q}_h$ and $\ubar{Q}_h$ analogous to those used in Optimistic ULCB, but which are updated online via a Q-learning update, rather than being recomputed at each episode under the current model.  We can extend Strategic ULCB to the model-free case in much the same way, defining the current exploration and evaluation policies as
    \begin{align}
        \mu^{k}_{h}(s), \tilde{\nu}^{k}_h(s) &= \text{Nash}(\bar{Q}_h(s,\cdot,\cdot)) \\
        \tilde{\mu}^{k}_{h}(s), \tilde{\nu}^{k}_h(s) &= \text{Nash}(\ubar{Q}_h(s,\cdot,\cdot))
    \end{align}
    and updating the value function bounds as
    \begin{align}
        \bar{V}^{k}_h(s) &= \mu^{k}_{h}(s)^{\top} \bar{Q}_h(s,\cdot,\cdot) \tilde{\nu}^{k}_h(s) \\
        \ubar{V}^{k}_h(s) &= \tilde{\mu}^{k}_{h}(s)^{\top} \ubar{Q}_h(s,\cdot,\cdot) \tilde{\nu}^{k}_h(s)
    \end{align}
    Like Optimistic Nash-Q, Strategic Nash-Q recomputes the policies and value function bounds for the current state after the $\bar{Q}_h$ and $\ubar{Q}_h$ are for the current state and action.  We provide the pseudocode for Optimistic Nash-Q in Appendix C, and for Strategic Nash-Q in Appendix D.

\section{Experiments}
    \label{experiments}
    
    \begin{figure}[ht]
        \centering
        \includegraphics[width=0.7\linewidth]{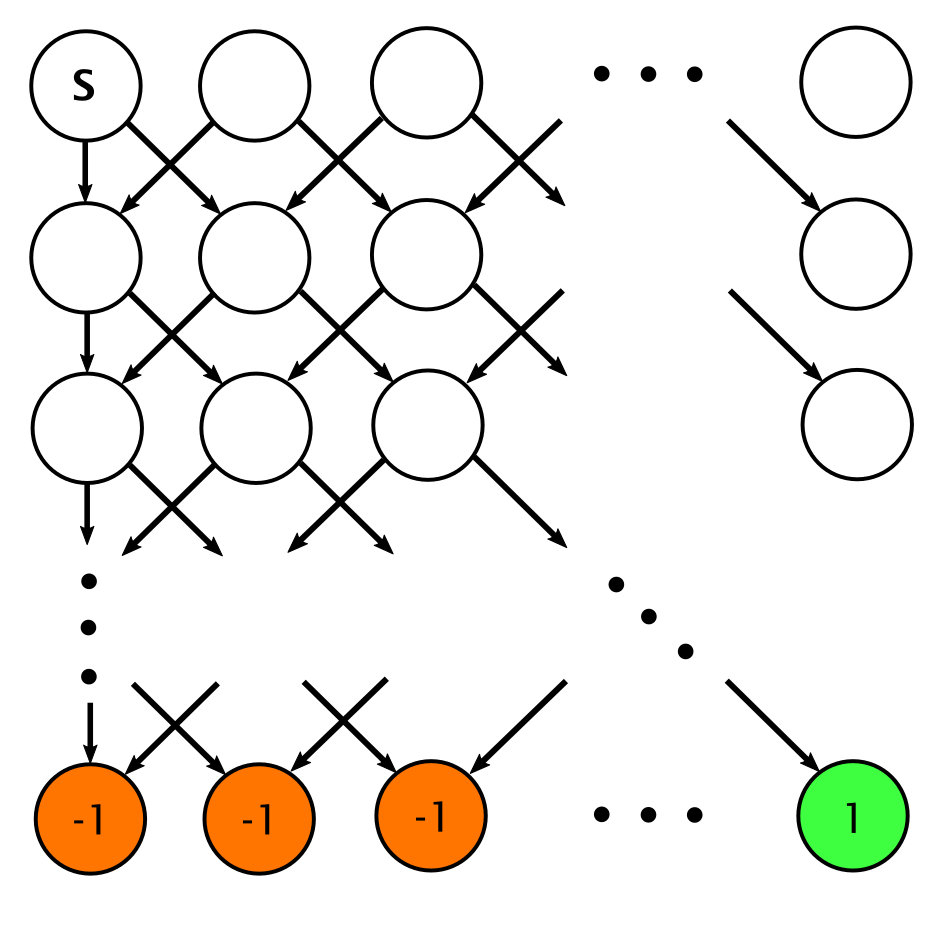}
        \caption{The $k \times k$ deep-sea task.  The player always starts in the state marked ``S''.  To reach the goal state, the player must move right for $k-1$ steps.}
        \label{fig:deep_sea}
    \end{figure}
    
    \begin{figure*}[t]
        \centering
        \begin{subfigure}[b]{0.32\textwidth}
            \caption{}
            \includegraphics[width=\textwidth]{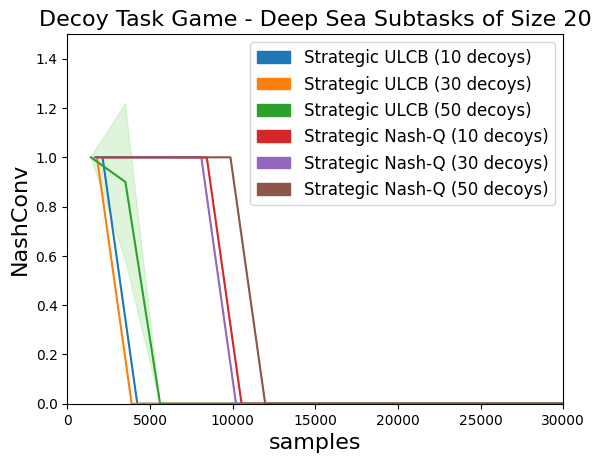}
            \label{fig:decoy_strategic_20}
        \end{subfigure}
        \begin{subfigure}[b]{0.32\textwidth}
            \caption{}
            \includegraphics[width=\textwidth]{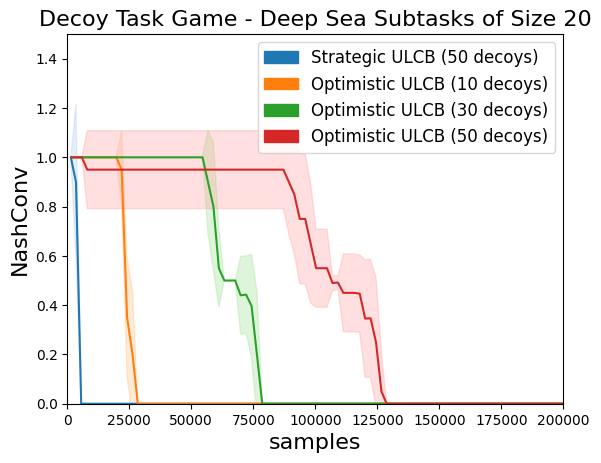}
            \label{fig:decoy_ulcb_20}
        \end{subfigure}
        \begin{subfigure}[b]{0.32\textwidth}
            \caption{}
            \includegraphics[width=\textwidth]{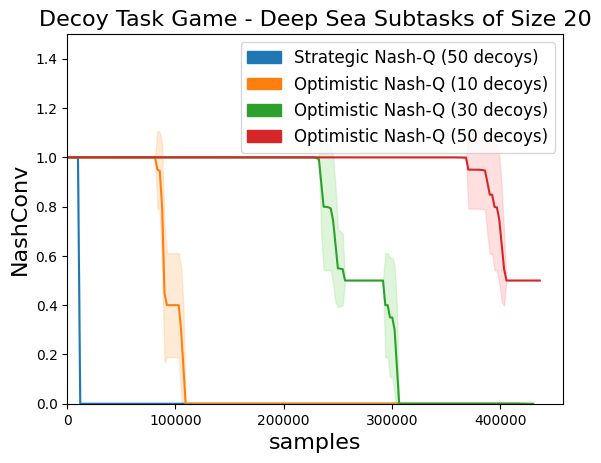}
            \label{fig:decoy_nash_q_20}
        \end{subfigure}

        \caption{Comparisons between Strategic ULCB, Strategic Nash-Q and their optimistic counterparts on decoy task games with deep-sea sub-tasks of varying sizes. Shows the $\text{NashConv}$ loss, with zero corresponding to the point where the game has been solved.  Error bars show standard deviations over 10 game instances.}
        \label{fig:decoy_deep_sea}
    \end{figure*}
    
    \begin{figure}[t]
        \centering
        \begin{subfigure}[b]{0.8\columnwidth}
            \caption{}
            \includegraphics[width=\textwidth]{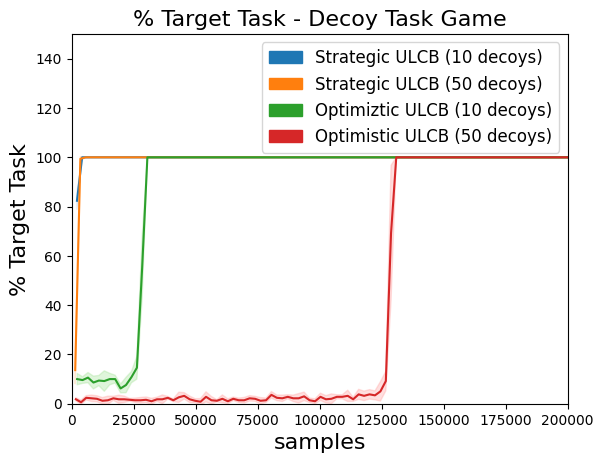}
            \label{fig:decoy_ulcb_counts_20}
        \end{subfigure}
        \begin{subfigure}[b]{0.8\columnwidth}
            \caption{}
            \includegraphics[width=\textwidth]{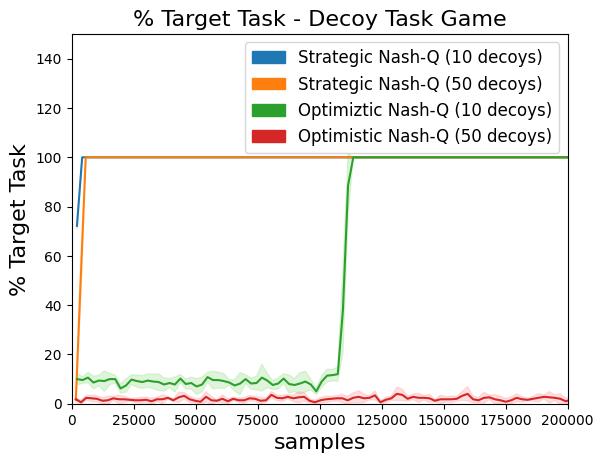}
            \label{fig:decoy_nash_q_counts_20}
        \end{subfigure}
        \caption{Comparisons between Strategic ULCB, Strategic Nash-Q and their optimistic counterparts on decoy task games with deep-sea sub-tasks.  Shows the percentage of episodes per iteration that explore the target task.  Error bars show standard deviations over 5 games.}
        \label{fig:decoy_deep_sea_counts}
    \end{figure}
    
    In this section, we compare Strategic ULCB and Strategic Nash-Q against Optimistic ULCB and Nash-Q, as well as Independent Q-learning.  To highlight the impact of strategically efficient exploration on sample complexity, we first present results in a version of the decoy task game (Figure~\ref{fig:decoy_game}).  We will demonstrate that, as the number of strategically irrelevant decoy tasks increases, so too do the advantages of Strategic ULCB and Nash-Q over the alternatives.  We also evaluate these algorithms on a set of randomly generated turn-based games, to demonstrate the value of strategic exploration in much more general settings.
    
    \subsection{Algorithms}
    \label{baselines}
    
    The fact that Strategic ULCB and Nash-Q define separate evaluation policies could give them an unfair advantage over the Optimistic baseline algorithms.  To provide a fair comparison between these approaches, we therefore modify Optimistic ULCB and Nash-Q to compute \emph{pessimistic} evaluation policies $\tilde{\mu}^{k}$ and $\tilde{\nu}^{k}$ a
    \begin{align}
        \mu, \tilde{\nu}^{k}_h(s) &= \text{Nash}(\bar{Q}^{k}_h(s,\cdot,\cdot)) \\
        \tilde{\mu}^{k}_h(s), \nu &= \text{Nash}(\ubar{Q}^{k}_h(s,\cdot,\cdot))
    \end{align}
    where $\bar{Q}^{k}_h$ and $\ubar{Q}^{k}_h$ are the upper and lower confidence bounds maintained by each algorithm.  These policies correspond to each player maximizing their expected return in the worst plausible case.  Note that existing $\text{NashConv}$ regret bounds for Optimistic Nash-Q only apply to a complex, non-stationary mixture of the exploration policies $\mu^{k}_h$ and $\nu^{k}_h$.  In these experiments, however, the $\text{NashConv}$ loss is computed for the most recent values of the evaluation policies $\tilde{\mu}^{k}$ and $\tilde{\nu}^{k}$.  We also note that, as our experiments are conducted in alternating move games, the computation of equilibrium strategies for each state reduces to a simple maximization problem over the actions for the current player.
    
    \paragraph{Independent Q-Learning}
    
    Additionally, we compare against a learner that trains by running two independent instances of tabular Q-learning against one another.  While this approach is not guaranteed to solve a Markov game, the use of independent Q-learning (IQL) has historically proven successful in some multi-agent settings~\citep{tan1993multi,tesauro1994td}.  In these experiments, we optimistically initialize the Q-function estimates for each learner to their maximum possible return $H$, which means that both learners engage in optimistic exploration in much the same way that Optimistic Nash-Q does, but without explicit coordination between the learners.  Like Optimistic ULCB and Nash-Q, each Q-learner maintains a separate evaluation policy based on a separate, pessimistically initialized Q-function.
    
    \paragraph{Hyper-parameters}
    
    For Strategic and Optimistic ULCB, the only hyperparameter that needs to be defined is the exploration bonus $\beta_t$ (Strategic and Optimistic Nash-Q require this parameter as well).  For our deterministic environments, however, we can set $\beta_t = 0$ for all $t$.  For Optimistic and Strategic Nash-Q (as well as independent Q-learning), we set the learning rate $\alpha_t = \frac{H + 1}{H + t}$, where $t = N^{k}_h(s,a,b)$ ($t = N^{k}_h(s,a)$ for IQL), the theoretically justified value which proved reliable in practice~\citep{jin2018qlearning,bai2020nash_q}. For IQL, we found empirically that using $\epsilon$-greedy exploration (in addition to optimistic exploration) led to better performance, with $\epsilon=0.05$ being most effective.  Code and instructions for reproducing these experiments is available at: \url{https://github.com/microsoft/strategically_efficient_rl}

    \subsection{Decoy Task Games}
    \label{decoy_task}

    To demonstrate advantage of strategically efficient exploration, we first evaluate Strategic ULCB and Strategic Nash-Q on instances of the decoy task game, illustrated in Figure~\ref{fig:decoy_game}.  The challenge for exploration in these games is the tendency of the \emph{decoy} tasks to distract algorithms that explore without regard for the adversarial nature of the game.  These games are representative of the broader class of two-player zero sum games in which the bulk of the state space is strategically irrelevant, that is, it does not need to be explored to find an equilibrium solution. To understand the impact of such irrelevant states, we consider games with a single target task, but varying numbers of decoy tasks.  To keep rewards normalized in $[0,1]$, we modify the payoff structure shown in Figure~\ref{fig:decoy_game} such that a max-player loss corresponds to a max-player reward of 0, and a tie a reward of $\sfrac{1}{2}$.  Each decoy and target task is separate (solving one does not help the learner solve the others), and we also randomize the the index of the action leading to the target task.
    
    \subsubsection{Deep Sea Sub-Task}
    
    In our experiments with decoy task games, both the target and decoy tasks are instances of the \emph{deep sea} environment~\citep{osband2019exploration,osband2019bsuite}.  We chose the deep sea environment as it is specifically designed to be difficult to solve using simple exploration strategies such as $\epsilon$-greedy, while being reliably solved using count-based exploration bonuses or optimistic initialization of the value function.  The deep sea environment (Figure~\ref{fig:deep_sea}) is an $n \times n$ grid of states, with the initial state in the top-left corner, and the goal state in the bottom right corner.  The player moves down, to the left or right, at each step, and to reach the goal, the player must go right for $n - 1$ steps.  For large $n$, random action selection will have a very small probability of reaching the goal.  The use of instances of the deep sea environment as target and decoy tasks in the decoy task game leads to a task for which efficient exploration is essential, but the naive application of single-agent exploration mechanisms perform poorly when there are a large number of decoy tasks.  This combination is therefore ideal for evaluating the strategic efficiency of a learning algorithm.
    
    \subsubsection{Decoy Task Game Results}
    \label{decoy_task_results}

    \begin{figure*}[t]
        \centering
        \begin{subfigure}[b]{0.32\textwidth}
            \caption{}
            \includegraphics[width=\textwidth]{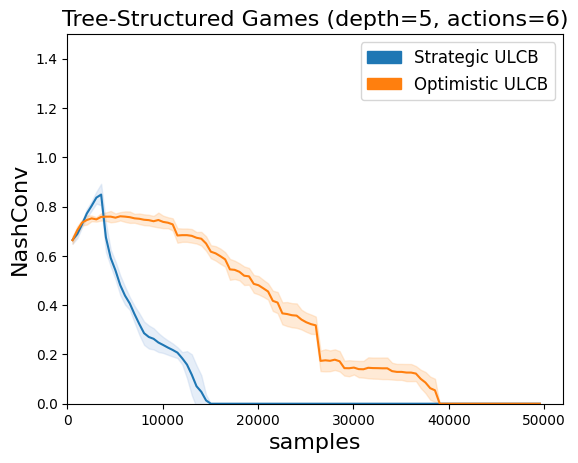}
            \label{fig:tree_5_6_ulcb}
        \end{subfigure}
        \begin{subfigure}[b]{0.32\textwidth}
            \caption{}
            \includegraphics[width=\textwidth]{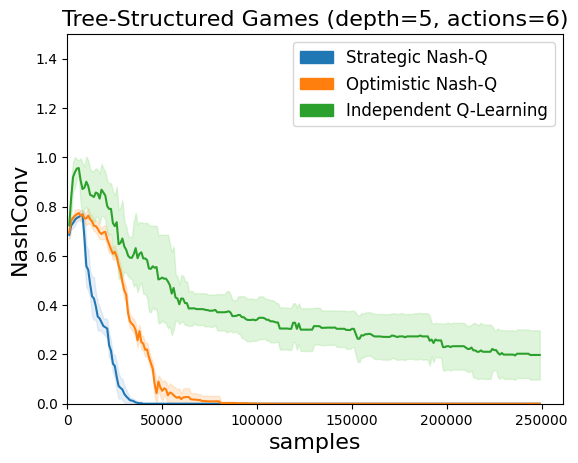}
            \label{fig:tree_5_6_nash_q}
        \end{subfigure}
        \begin{subfigure}[b]{0.32\textwidth}
            \caption{}
            \includegraphics[width=\textwidth]{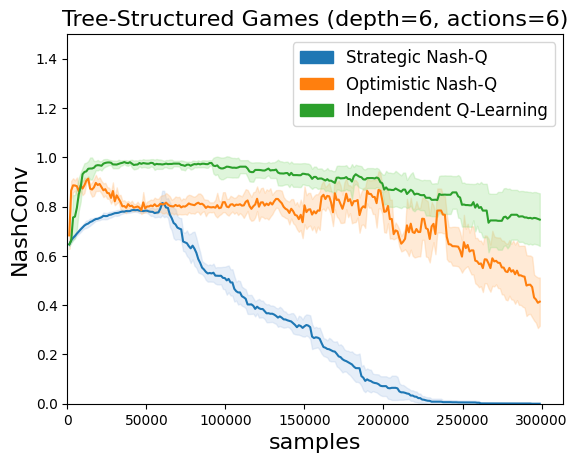}
            \label{fig:tree_6_6_nash_q}
        \end{subfigure}
        \caption{Comparisons of Strategic-ULCB and Strategic-Nash-Q against their optimistic counterparts and IQL, in tree-structured, alternating move games with randomly generated payoffs.  Shows the $\text{NashConv}$ loss, with zero corresponding to the point where the game has been solved.  Error bars show standard deviation over 10 randomly generated game instances.}
        \label{fig:tree_game}
    \end{figure*}
    
    Figure~\ref{fig:decoy_deep_sea} shows a set of comparisons between Strategic ULCB and Strategic Nash-Q against their optimistic counterparts, with pessimistic evaluation policies, in several instances of the the decoy task game.  We compare these algorithms on instances with 10, 30 and 50 decoy tasks, where both target and decoy sub-tasks are instances of the 20x20 deep sea environment.  Figure~\ref{fig:decoy_strategic_20} shows that, in terms of the $\text{NashConv}$ loss, the performance of Strategic ULCB and Strategic Nash-Q, is largely insensitive to the number of decoy tasks.  In contrast, Figures~\ref{fig:decoy_ulcb_20} and~\ref{fig:decoy_nash_q_20} show that the number of samples required for Optimistic ULCB and Optimistic Nash-Q to solve the game (where the loss goes to zero) grows roughly proportionately with the number of decoy tasks.  For clarity of presentation, Figures~\ref{fig:decoy_ulcb_20} and~\ref{fig:decoy_nash_q_20} only show the performance of Strategic ULCB and Strategic Nash-Q for the most difficult case with 50 decoys, which is nonetheless significantly better than that of the optimistic algorithms for even the easiest case with 10 decoys.  As expected, the model-based ULCB algorithms are more sample-efficient than their model-free counterparts.
    
    These results are consistent with our hypothesis that the strategically efficient algorithms will be able to quickly recognize that for any decoy task, the best return the max-player can expect against a rational opponent would be $\sfrac{1}{2}$, and so will prioritize solving the target instance of the deep-sea sub-task.  We can see this behavior in Figure~\ref{fig:decoy_strategic_20}, where Strategic ULCB and Nash-Q take slightly longer to solve games with more decoys, corresponding to the time required to determine that each decoy task is irrelevant.  To further support this hypothesis, in Figure~\ref{fig:decoy_deep_sea_counts} we show the percentage of episodes in which each algorithm explored the target task.  We can see that Strategic ULCB and Nash-Q concentrate exploration on solving the target task much more quickly than Optimistic ULCB and Nash-Q, which waste time attempting to solve each decoy task.

    \subsection{Tree-Structured Games}
    \label{random_tree}
    
    While strategically efficient exploration has a dramatic impact on performance in the decoy task game, we can also show that it can have a significant impact on performance in much more general classes of games.  In this section, we evaluate Strategic ULCB and Strategic Nash-Q in a space of tree-structured, alternating move games, where, for each random game instance, the max-player rewards for each terminal state is drawn from the uniform distribution over $[0,1]$.  While there are no states in these games that are designated as being strategically irrelevant, we can nonetheless bound the plausible return each player can guarantee from a given state without knowing the payoffs of all terminal states reachable from that state.  Strategic exploration may therefore still be beneficial, if it can prioritize states for which a strong adversary policy has not yet been identified.
    
    We consider alternating-move games of depth 5 and 6, with either 5 or 6 actions available to the active player in each state.  Figure~\ref{fig:tree_5_6_ulcb} compares the average performance of Strategic and Optimistic ULCB in solving games of depth 5 with 6 actions per state, where Strategic ULCB has a clear advantage in how fast its $\text{NashConv}$ loss converges to zero.  In Figures~\ref{fig:tree_5_6_nash_q} and~\ref{fig:tree_6_6_nash_q}, Strategic Nash-Q shows an advantage over Optimistic Nash-Q (as well as independent Q-learning) in games of depth 5 and 6, with Strategic Nash-Q having a large advantage over the alternatives for games of depth 6.  
    
\section{Conclusion}
    
    Reducing sample complexity will be critical if reinforcement learning is to see widespread use in solving real-world problems, particularly for tasks that involve interaction between multiple agents.  Here we have considered approaches to exploration in competitive multi-agent tasks, and have shown that the use of \emph{strategically efficient} exploration mechanisms can significantly reduce sample complexity relative to non-strategic, optimistic mechanisms.  We have presented novel, strategically efficient reinforcement learning algorithms for finite Markov games, and demonstrated that they can be significantly more sample efficient than their optimistic counterparts in challenging exploration games, while preserving the same sample complexity guarantees as existing approaches across all possible games.
    
    While this work is limited to Markov games with small, finite state and action spaces, the concept of strategically efficient exploration can be applied to games with infinite state and action spaces.  Future work would focus on the development of strategically efficient algorithms that are compatible with the use of function approximation.  The finite algorithms developed in this work may serve as the basis for strategically efficient alternatives to existing frameworks for deep multi-agent RL.  Future theoretical work would seek to extend the notion of strategic efficiency to $n$-player general-sum games, and games with imperfect information.

\begin{acknowledgements}

    We would like to thank Akshay Krishnamurthy for his valuable feedback during the development of this work.

\end{acknowledgements}

\bibliography{references}

\begin{appendices}

\include{appendices}

\end{appendices}

\end{document}

%% file: appendices.tex
\section{Proof of Theorem 4.1}
\label{apx:strategic_uclb_regret}

Our proof of regret bounds for Strategic ULCB is very similar to the proof for the regret bounds of Optimistic ULCB~\citep{bai2020ulcb}.  The key difference is that we need to show that the tighter confidence bounds maintained by Strategic ULCB still constrain the exploitability of the evaluation policies $\tilde{\mu}^{k}$ and $\tilde{\nu}^{k}$ (Lemma~\ref{lem:strategic_confidence_bounds}), and that these confidence bounds still converge under our exploration policies.  We also directly bound the $L_1$ error of the transition model, which leads to a somewhat simpler proof, and helps us better understand the nature of the confidence sets that Strategic ULCB implicitly maintains.

\begin{lemma}
    \label{lem:model_error}
    For a given $K \geq 3$ and $\delta > 0$, define $\beta_t$ as
    \begin{equation}
        \label{eqn:formal_beta}
        \beta_t = H \sqrt{\frac{2\left[\vert S\vert \ln(KH\vert S\vert\vert A\vert\vert B \vert / \delta)\right]}{t}}
    \end{equation}
    then with probability at least $1 - \delta$, for all $k \in [K]$, $h \in H$, $s \in S_h$, $a \in A_{h,s}$ and $b \in B_{h,s}$, and for all $V \in [0,H]^{\vert S \vert}$, we have
    \begin{equation}
        \label{eqn:model_bound}
        \left\vert \hat{P}^{k}_h(s,a,b)^{\top}V - P(s,a,b)^{\top}V\right\vert \leq \beta_t
    \end{equation}
    for $t = N^{k}_h(s,a,b)$.
\end{lemma}

\begin{proof}
    When $N^{k}_h(s,a,b) = 0$, Equation~\ref{eqn:model_bound} holds trivially as $\beta_t = \infty$.    Otherwise, we can apply the well known bound on the $L_1$ error of an empirical distribution due to~\cite{weissman2003error} to show that
    \begin{multline}
        \text{Pr}\left\{ \Vert \hat{P}^{k}_h(s,a,b) - P(s,a,b) \Vert_1 \geq \epsilon \right\} \leq \\ (2^{\vert S\vert} - 2)\exp\{-N^{k}_h(s,a,b)\frac{\epsilon}{2}\}
    \end{multline}
    Note that, for all $V \in [0,H]^{\vert S \vert}$
    \begin{multline}
        \vert \hat{P}^{k}_h(s,a,b)^{\top} V - P(s,a,b)^{\top} V \vert \leq \\ H \Vert \hat{P}^{k}_h(s,a,b) - P(s,a,b) \Vert_1
    \end{multline}
    and so for $t = N^{k}_h(s,a,b)$ and $\beta_t$ defined according to Equation~\ref{eqn:formal_beta} we therefore have
    \begin{multline}
        \text{Pr} \left\{ \exists V, \vert \hat{P}^{k}_h(s,a,b)^{\top} V - P(s,a,b)^{\top} V \vert \geq \beta_t \right\} \leq \\ \frac{\delta}{KH\vert S\vert\vert A\vert\vert B \vert}
    \end{multline}
    Taking the union bound over $k$, $h$, $s$, $a$ and $b$ yields the desired result.
\end{proof}

For games with deterministic transitions, $\hat{P}^{k}_h(s,a,b) = P(s,a,b)$ whenever $N^{k}_h(s,a,b) > 0$, and so Equation~\ref{eqn:model_bound} will hold even for $\beta_t = 0$, which is the value we use for our experiments in deterministic games.  We can now show that our confidence bounds $\bar{V}^{k}_{h}$ and $\ubar{V}^{k}_{h}$ not only constrain the value of the game at each state, but also bound the exploitability of our evaluation policies $\tilde{\mu}^{k}$ and $\tilde{\nu}^{k}$.

\begin{lemma}
    \label{lem:strategic_confidence_bounds}
    When Strategic-ULCB is run with $\beta_t$ as defined in Equation~\ref{eqn:formal_beta}, the for all $k \in [K]$, $h \in H$ and $s \in S_h$, we have
    \begin{align}
        \bar{V}^{k}_h(s) &\geq \sup_{\mu}V^{\mu,\nu^{k}}_h(s) \\
        \ubar{V}^{k}_h(s) &\leq \inf_{\nu}V^{\mu^{k},\nu}_h(s)
    \end{align}
    with probability at least $1 - \delta$.
\end{lemma}

\begin{proof}
    For each $k \in [K]$ we prove this by induction on $h$.  We will only show the proof for the upper bound, as the proof for the lower bound is symmetric.  Assume that for some $h \in [H]$ we have, for all $s \in S_h$
    \begin{equation}
        \label{eqn:bound_assumption}
        \bar{V}^{k}_{h+1}(s) \geq \sup_{\mu}V^{\mu,\nu^{k}}_{h+1}(s)
    \end{equation}
    By Lemma~\ref{lem:model_error}, Equation~\ref{eqn:model_bound} will hold simultaneously for all $k$, $h$, $s$, $a$ and $b$ with probability at least $1-\delta$, and so when $N^{k}_h(s,a,b) > 0$, we have
    \begin{align}
        \label{eqn:q_bound}
        \bar{Q}^{k}_h(s,a,b) &= \hat{R}^{k}_h(s,a,b) + \hat{P}^{k}_h(s,a,b)\bar{V}^{k}_{h+1} + \beta_t \\
        &\geq R(s,a,b) + \hat{P}^{k}_h(s,a,b)\bar{V}^{k}_{h+1}  \\
        &\geq R(s,a,b) + P(s,a,b) \sup_{\mu}V^{\mu,\nu^{k}}_{h+1} \\
        &= \sup_{\mu}Q^{\mu,\nu^{k}}_{h}(s,a,b)
    \end{align}
    where the $t = N^{k}_h(s,a,b)$, and the first inequality also uses the fact that $\hat{R}^{k}_h(s,a,b) = R(s,a,b)$ when $N^{k}_h(s,a,b) > 0$.  When $N^{k}_h(s,a,b) = 0$, Equation~\ref{eqn:q_bound} holds trivially, as $\bar{Q}^{k}_h(s,a,b) = H$.  By the definition of $\bar{V}^{k}_{h+1}(s)$, we then have
    \begin{align}
        \bar{V}^{k}_{h+1}(s) - \sup_{\mu}V^{\mu,\nu^{k}}_h(s) &= \mu^{k}_h(s)^{\top} \bar{Q}^{k}_h(s,\cdot,\cdot)\tilde{\nu}^{k}_h(s) \\ 
        & - \max_{a \in A_{h,s}}\sup_{\mu}Q^{\mu,\nu^{k}}_h(s,a,\cdot)\tilde{\nu}^{k}_h(s) \\
        &\geq \mu^{k}_h(s)^{\top} \bar{Q}^{k}_h(s,\cdot,\cdot)\tilde{\nu}^{k}_h(s) \\
        & - \max_{a \in A_{h,s}}\sup_{\mu}Q^{k}_h(s,a,\cdot)\tilde{\nu}^{k}_h(s) \\
        &= 0
    \end{align}
    which proves the inductive step.  The first inequality follows directly from Equation~\ref{eqn:q_bound}, while the second inequality follows from the fact that $(\mu^{k}_h(s),\tilde{\nu}^{k}_h(s))$ for a Nash equilibrium of the matrix game defined by $\bar{Q}^{k}_h(s,\cdot,\cdot)$, and so $\mu^{k}_h(s)$ is a best-response to $\tilde{\nu}^{k}_h(s)$ under $\bar{Q}^{k}_h(s,\cdot,\cdot)$.  Finally, we can see that Equation~\ref{eqn:bound_assumption} holds trivially for $h = H+1$, where we implicitly assume that $\bar{V}^{k}_{h}(s) = \sup_{\mu}V^{\mu,\nu^{k}}_{h}(s) = 0$, which concludes the proof.
\end{proof}

Lemma~\ref{lem:strategic_confidence_bounds} will be sufficient to prove Theorem~\ref{thm:strategic_ulcb_complexity} and bound the $\text{NashConv}$ regret of the evaluation policies $\tilde{\mu}^k$ and $\tilde{\nu}^k$.  The remainder of the proof will closely follow the proof for Optimistic ULCB given by~\cite{bai2020ulcb}, with slight modifications to account for the presence of separate exploration and evaluation policies.

\begin{proof}[Proof of Theorem~\ref{thm:strategic_ulcb_complexity}]

We begin with the definition of the $\text{NashConv}$ regret
\begin{equation}
    \text{Regret}(K) = \sum^{K}_{k = 1} \sup_{\mu}V^{\mu,\nu^{k}}_1(s_1) - \inf_{\nu}V^{\mu^{k},\nu}_1(s_1)
\end{equation}
for any $k \in [K]$ and $h \in [H]$, we have 
\begin{align}
    & \sup_{\mu}V^{\mu,\nu^{k}}_h(s^{k}_h) - \inf_{\nu}V^{\mu^{k},\nu}_h(s^{k}_h) \\
    &\leq \bar{V}^{k}_h(s^{k}_h) - \ubar{V}^{k}_h(s^{k}_h) \\
    &= \mu^{k}_h(s^{k}_h)^{\top} \bar{Q}^{k}_h(s^{k}_h,\cdot,\cdot) \tilde{\nu}^{k}_h(s^{k}_h) - \tilde{\mu}^{k}_h(s^{k}_h)^{\top} \ubar{Q}^{k}_h(s^{k}_h,\cdot,\cdot) \nu^{k}_h(s^{k}_h) \\
    &\leq \mu^{k}_h(s^{k}_h)^{\top} \bar{Q}^{k}_h(s^{k}_h,\cdot,\cdot) \nu^{k}_h(s^{k}_h) - \mu^{k}_h(s^{k}_h)^{\top} \ubar{Q}^{k}_h(s^{k}_h,\cdot,\cdot) \nu^{k}_h(s^{k}_h) \\
    &= \mu^{k}_h(s^{k}_h)^{\top} \left[\bar{Q}^{k}_h(s^{k}_h,\cdot,\cdot) - \ubar{Q}^{k}_h(s^{k}_h,\cdot,\cdot) \right] \nu^{k}_h(s^{k}_h)
\end{align}
where the first inequality follows from Lemma~\ref{lem:strategic_confidence_bounds}, while the second follow from the fact that $\tilde{\mu}^{k}$ and $\tilde{\nu}^{k}$ are best responses, and so changing to the optimistic strategies $\mu^{k}$ and $\nu^{k}$ can only increase the width of the confidence interval.  We can decompose the last term as
\begin{align}
    &\mu^{k}_h(s^{k}_h)^{\top} \left[\bar{Q}^{k}_h(s^{k}_h,\cdot,\cdot) - \ubar{Q}^{k}_h(s^{k}_h,\cdot,\cdot) \right] \nu^{k}_h(s^{k}_h) \\
    &= \left[\bar{Q}^{k}_h - \ubar{Q}^{k}_h\right](s^{k}_h,a^{k}_h,b^{k}_h) + \xi^{k}_h  \\
    &= \hat{P}^{k}_h(s^{k}_h,a^{k}_h,b^{k}_h)^{\top}\left[\bar{V}^{k}_{h+1} - \ubar{V}^{k}_{h+1}\right] + + 2\beta^{k}_h + \xi^{k}_h \\
    &= P(s^{k}_h,a^{k}_h,b^{k}_h)^{\top}\left[\bar{V}^{k}_h - \ubar{V}^{k}_h\right] + 4\beta^{k}_h + \xi^{k}_h \\
    &= \left[\bar{V}^{k}_{h+1} - \ubar{V}^{k}_{h+1}\right](s^{k}_{h+1}) + \zeta^{k}_h + 4\beta^{k}_h + \xi^{k}_h
\end{align}
where $\beta^{k}_h = \beta_t$ for $t = N^{k}_h(s,a,b)$.  The terms $\xi^{k}_h$ and $\zeta^{k}_h$ are defined as
\begin{align}
    \xi^{k}_h &= \text{E}_{a,b \sim \mu^{k}_h(s^{k}_h),\nu^{k}_h(s^{k}_h)}\left[\bar{Q}^{k}_h - \ubar{Q}^{k}_h\right](s^{k}_h,a,b) \\ 
    & \hspace{.6cm} - \left[\bar{Q}^{k}_h - \ubar{Q}^{k}_h\right](s^{k}_h,a^{k}_h,b^{k}_h) \\
    \zeta^{k}_h &= \text{E}_{s \sim P(s^{k}_h,a^{k}_h,b^{k}_h)}\left[\bar{V}^{k}_{h+1} - \ubar{V}^{k}_{h+1}\right](s) \\ 
    & \hspace{.6cm} - \left[\bar{V}^{k}_{h+1} - \ubar{V}^{k}_{h+1}\right](s^{k}_{h+1}) \\
\end{align}
Here $\xi^{k}_h$ and $\zeta^{k}_h$ are not i.i.d., but the sequences of their partial sums over $k$ and $h$ are martingales, and so by the Azuma-Hoeffding inequality
\begin{align}
    \sum_{k = 1}^{K}\sum_{h = 1}^{H} \xi^{k}_h &\leq \sqrt{2KH^{3}\ln{\frac{1}{\delta}}} \\
    \sum_{k = 1}^{K}\sum_{h = 1}^{H} \zeta^{k}_h &\leq \sqrt{2KH^{3}\ln{\frac{1}{\delta}}}
\end{align}
we then have
\begin{align}
    &\sum^{K}_{k = 1} \sup_{\mu}V^{\mu,\nu^{k}}_1(s_1) - \inf_{\nu}V^{\mu^{k},\nu}_1(s_1) \\
    &\leq \sum^{K}_{k = 1}\left[\bar{V}^{k}_h(s^{k}_1) - \ubar{V}^{k}_1(s^{k}_1)\right] \\
    &\leq \sum^{K}_{k = 1}\sum^{H}_{h=1} \left[4\beta^{k}_h + \xi^{k}_h + \zeta^{k}_h \right]
\end{align}
For $\beta^{k}_h$ we have
\begin{align}
    \sum^{K}_{k = 1}\sum^{H}_{h=1} \beta^{k}_h &= C \sum^{H}_{h=1} \sum_{s \in S_h}\sum_{a \in A_{h,s}}\sum_{b \in B_{h,s}}\sum_{t=1}^{N^{K}_h(s,a,b)} \frac{1}{\sqrt{t}} \\
    &\leq \sqrt{K H^{2} \vert S\vert \vert A\vert \vert B\vert}
\end{align}
by the Cauchy-Schwarz inequality, where
\begin{equation}
    C = \sqrt{2H^{2}\vert S\vert \ln(KH\vert S\vert\vert A\vert\vert B \vert / \delta)}
\end{equation}
finally, this gives us
\begin{align}
    &\sum^{K}_{k = 1}\sum^{H}_{h=1} \left[4\beta^{k}_h + \xi^{k}_h + \zeta^{k}_h \right] \\
    &\leq 4\sqrt{2KH^{4}\vert S\vert^{2} \vert A\vert \vert B\vert\ln(KH\vert S\vert\vert A\vert\vert B \vert / \delta)} + 2\sqrt{2KH^{3}\ln{\frac{1}{\delta}}} \\
    &\leq 6\sqrt{2KH^{4}\vert S\vert^{2} \vert A\vert \vert B\vert\ln(KH\vert S\vert\vert A\vert\vert B \vert / \delta)}
\end{align}
which completes the proof.
\end{proof}

\section{Proof of Theorem 4.2}
\label{apx:strategic_ulcb_efficiency}

We prove~\ref{thm:strategic_ulcb_efficiency} for the max-player's exploration strategy $\mu^{k}$ only, as the proof for the min-player's strategy is symmetric.  We first show that the upper bounds $\bar{V}^{k}_h$ and $\bar{Q}^{k}_h$ can always be achieved for some game in $D_k$.

\begin{lemma}
    \label{lem:existence}
    At each episode $k$, there exists a game $G \in D_k$ such that the upper confidence bounds $\bar{V}^{k}$ and $\bar{Q}^{k}_h$ computed by Strategic-ULCB for $\beta_t = 0$ satisfy
    \begin{align}
        \label{eqn:v_existence}
        \bar{V}^{k}_h(s) = \sup_{\mu}\inf_{\nu}V^{\mu,\nu}_{G,h}(s) \\
        \label{eqn:q_existence}
        \bar{Q}^{k}_h(s, a, b) = \sup_{\mu}\inf_{\nu}Q^{\mu,\nu}_{G,h}(s, a, b)
    \end{align}
    for all $h \in [H]$ and $s \in S_h$, and $a \in A_{h,s}$ or $b \in B_{h,s}$.
\end{lemma}

\begin{proof}
We prove this by induction on $h$.  Assume that for some $k \geq 1$, $h \ in [H]$, there exists a game $G \in D_k$ such that
\begin{equation}
    \label{eqn:existence_assumption}
    \bar{V}^{k}_{h+1}(s) = \sup_{\mu}\inf_{\nu}V^{\mu,\nu}_{G,h+1}(s)
\end{equation}
for all $s \in S_{h+1}$.  For each $s \in S_h$, $a \in A_{h,s}$, and $b \in B_{h,s}$, if $(h, s,a,b) \in \mathcal{H}_t$, then since $G \in D^k$ we will have $\hat{R}^{k}_h(s,a,b) = R_{G,h}(s,a,b) = R_h(s,a,b)$ and $\hat{P}^{k}_h(s,a,b) = P_{G,h}(s,a,b) = P_h(s,a,b)$, and so
\begin{align}
    \bar{Q}^{k}_h(s,a,b) &= R_{G,h}(s,a,b) + P_{G,h}(s,a,b)^{\top}\bar{V}^{k}_{G,h+1} \\
    &= \sup_{\mu}\inf_{\nu}Q^{\mu,\nu}_{G,h}(s,a,b)
\end{align}
On the other hand, if $(h, s,a,b) \notin \mathcal{H}_t$, then we have $\bar{Q}^{k}_h(s,a,b) = H$.  In this case, there exists a game $G' \in D^k$ that is equivalent to $G$ for all $h' \geq h$, but for which $P_{G',h}(s,a,b,s^{*}) = 1$, and $R_{h}(s,a,b) = H$, where $s^{*}$ is our hypothetical absorbing state with reward $0$ for all actions and time steps.  Because transition distributions can be selected independently of one another for each $s$, $a$ and $b$, there exists $G' \in D_k$ such that $P_{G',h}(s,a,b,s^{*}) = 1$, and $R_{h}(s,a,b) = H$ for all $s \in S_h$, $a \in A_{h,s}$, and $b \in B_{h,s}$ where $(h,s,a,b) \notin \mathcal{h}_t$, such that $\bar{Q}^{k}_h(s,a,b) = \sup_{\mu}\inf_{\nu}Q^{\mu,\nu}_{G',h}(s,a,b)$.  We then have that
\begin{align}
    \bar{V}^{k}_h(s) &= \mu^{k}_h(s)^{\top}\bar{Q}^{k}_h(s,\cdot,\cdot)\tilde{\nu}^{k}_h(s) \\
    &= \sup_{\mu}\inf_{\nu}\bar{Q}^{k}_h(s,\cdot,\cdot) \\
    &= \sup_{\mu}\inf_{\nu}\bar{Q}^{k}_{G',h}(s,\cdot,\cdot)
\end{align}
Noting that Equation~\ref{eqn:existence_assumption} holds trivially for $h = H$, where we implicitly assume that $\bar{V}^{k}_{H+1} = V^{\mu,\nu}_{H+1} = 0$, this proves the lemma for all $h \in H$.
\end{proof}

To show that Strategic ULCB is strategically efficient for the max-player exploration policy, we need to show that, for some game $G \in D_k$, $\mu^{k}$ is the max player component of a Nash equilibrium of $G$.  

\begin{proof}[Proof of Theorem~\ref{thm:strategic_ulcb_efficiency}]
Let $G \in D_k$ be a game for which Equations~\ref{eqn:v_existence} and~\ref{eqn:q_existence} hold.  By Lemma~\ref{lem:existence}, such a game always exists.  We can prove that $\mu^k$ is a max-player component of an equilibrium of $G$ by induction on $h$.  Assume that, for $h \in [H]$ and for all $s \in S_h$
\begin{equation}
    \label{eqn:equilibrium_assumption}
    \mu^{k} \in \argmax_{\mu}\inf_{\nu}V^{\mu,\nu}_{G,h+1}(s)
\end{equation}
We then have that, for all $h \in H$, $s \in S_h$
\begin{align}
    \mu^{k}_h(s) &\in \argmax_{x}\inf_{y}x^{\top}\bar{Q}^{k}_h(s,\cdot, \cdot) y \\
    &= \argmax_{x}\inf_{y}x^{\top}\left[\sup_{\mu}\inf_{\nu}Q^{\mu,\nu}_{G,h}(s,\cdot, \cdot)\right] y \\
    &= \argmax_{x}x^{\top}\left[\sup_{\mu}\inf_{\nu}Q^{\mu,\nu}_{G,h}(s,\cdot, \cdot)\nu_h(s)\right] \\
    &= \argmax_{x}x^{\top}\inf_{\nu}V^{\mu^{k},\nu}_{G,h+1}(s)
\end{align}
where the last line implies that
\begin{equation}
    \mu^{k} \in \argmax_{\mu}\inf_{\nu}V^{\mu,\nu}_{G,h}(s)
\end{equation}
Noting that Equation~\ref{eqn:equilibrium_assumption} is implicitly satisfied for $h = H$, this concludes the proof for $\mu^{k}$.  Repeating this process for $\nu^{k}$ proves the result.
\end{proof}

\section{Strategic Nash-Q Algorithm}
\label{apx:nash_q}

\begin{algorithm}[h]
    \begin{algorithmic}
        \STATE \textbf{Inputs:} $\alpha_{t \geq 0}$, $\beta_{t \geq 0}$ \\
        \STATE \textbf{Initialize:} $\forall (h, s, a, b)$, $\bar{Q}_h(s,a,b) \leftarrow H$, $\ubar{Q}_h(s,a,b) \leftarrow 0$, $N_h(s,a,b) \leftarrow 0$, $\mu^{1}_h(s,a) \leftarrow \sfrac{1}{\vert A_{h,s}\vert}$, $\nu^{1}_h(s,a) \leftarrow \sfrac{1}{\vert B_{h,s}\vert}$.
        \FOR{episode $k = 1,\ldots,K$}
            \STATE observe $s^{k}_1$.
            \FOR{step $h=1,\ldots,H$}
                \STATE take action $a^{k}_h \sim \mu^{k}_{h}(s^{k}_h)$, $b^{k}_h \sim \nu^{k}_{h}(s^{k}_h)$.
                \STATE observe reward $r^{k}_h$, next state $s^{k}_{h+1}$.
                \STATE $N_h(s^{k}_h,a^{k}_h,b^{k}_h) \leftarrow N_h(s^{k}_h,a^{k}_h,b^{k}_h) + 1$
                \STATE $t \leftarrow N_h(s^{k}_h,a^{k}_h,b^{k}_h)$
                \STATE $\bar{Q}_h(s^{k}_h,a^{k}_h,b^{k}_h) \leftarrow \min\{(1 - \alpha_t)\bar{Q}_h(s^{k}_h,a^{k}_h,b^{k}_h) + \alpha_t (r^{k}_h + \bar{V}^{k}_{h+1}(s^{k}_{h+1}) + \beta_t),H\}$
                \STATE $\ubar{Q}_h(s^{k}_h,a^{k}_h,b^{k}_h) \leftarrow \max\{(1 - \alpha_t)\ubar{Q}_h(s^{k}_h,a^{k}_h,b^{k}_h) + \alpha_t (r^{k}_h + \ubar{V}^{k}_{h+1}(s^{k}_{h+1}) - \beta_t),0\}$
                \STATE $\mu^{k+1}_h(s^{k}_h), \tilde{\nu^{k+1}_h(s^{k}_h)} \leftarrow \text{Nash}(\bar{Q}_h(s^{k}_h,\cdot,\cdot))$
                \STATE $\mu^{k+1}_h(s^{k}_h), \tilde{\nu^{k+1}_h(s^{k}_h)} \leftarrow \text{Nash}(\bar{Q}_h(s^{k}_h,\cdot,\cdot))$
            \ENDFOR
        \ENDFOR
    \end{algorithmic}
    \caption{The Strategic Nash-Q algorithm.  Similar to Optimistic Nash-Q, Strategic Nash-Q maintains upper and lower bounds on the optimal value and Q-functions.  Unlike Optimistic Nash-Q, Strategic Nash-Q computes the max and min-player policies for each state independently, of one another, and updates its value function bounds under the assumption that the adversary acts pessimistically (optimizes the lower-bound on its expected return, rather than the upper bound).  Like Strategic ULCB, Strategic Nash-Q maintains separate evaluation policies $\mu^k$ and $\nu^k$.}
    \label{alg:nash_q}
\end{algorithm}

Algorithm~\ref{alg:nash_q} details the Strategic Nash-Q algorithm, which applies the strategically efficient updater rules of Strategic ULCB to the Optimistic Nash-Q algorithm of~\cite{bai2020nash_q}.  Here, the sequences of learning rates $\alpha_t$ and exploration bonuses $\beta_t$ are left as free hyperparameters that can be tuned to a specific task.  In our experimental results, we use the theoretically justified learning rate of $\alpha_t = \frac{H + 1}{H + t}$, while $\beta_t = \beta \frac{1}{\sqrt{t}}$, where $\beta$ is a task-specific hyperparameter.

%% file: arxiv.bbl
\begin{thebibliography}{21}
\providecommand{\natexlab}[1]{#1}
\providecommand{\url}[1]{\texttt{#1}}
\expandafter\ifx\csname urlstyle\endcsname\relax
  \providecommand{\doi}[1]{doi: #1}\else
  \providecommand{\doi}{doi: \begingroup \urlstyle{rm}\Url}\fi

\bibitem[Bai and Jin(2020)]{bai2020ulcb}
Yu~Bai and Chi Jin.
\newblock Provable self-play algorithms for competitive reinforcement learning.
\newblock In \emph{International Conference on Machine Learning}, pages
  551--560. PMLR, 2020.

\bibitem[Bai et~al.(2020)Bai, Jin, and Yu]{bai2020nash_q}
Yu~Bai, Chi Jin, and Tiancheng Yu.
\newblock Near-optimal reinforcement learning with self-play.
\newblock \emph{Advances in Neural Information Processing Systems}, 33, 2020.

\bibitem[Berner et~al.(2019)Berner, Brockman, Chan, Cheung, Debiak, Dennison,
  Farhi, Fischer, Hashme, Hesse, Józefowicz, Gray, Olsson, Pachocki, Petrov,
  de~Oliveira~Pinto, Raiman, Salimans, Schlatter, Schneider, Sidor, Sutskever,
  Tang, Wolski, and Zhang]{berner2019dota}
Christopher Berner, Greg Brockman, Brooke Chan, Vicki Cheung, Przemyslaw
  Debiak, Christy Dennison, David Farhi, Quirin Fischer, Shariq Hashme, Chris
  Hesse, Rafal Józefowicz, Scott Gray, Catherine Olsson, Jakub Pachocki,
  Michael Petrov, Henrique~Pondé de~Oliveira~Pinto, Jonathan Raiman, Tim
  Salimans, Jeremy Schlatter, Jonas Schneider, Szymon Sidor, Ilya Sutskever,
  Jie Tang, Filip Wolski, and Susan Zhang.
\newblock Dota 2 with large scale deep reinforcement learning.
\newblock \emph{CoRR}, 2019.

\bibitem[B{\"o}hmer et~al.(2019)B{\"o}hmer, Rashid, and
  Whiteson]{bohmer2019intrinsic}
Wendelin B{\"o}hmer, Tabish Rashid, and Shimon Whiteson.
\newblock Exploration with unreliable intrinsic reward in multi-agent
  reinforcement learning.
\newblock \emph{arXiv preprint arXiv:1906.02138}, 2019.

\bibitem[Burda et~al.(2018)Burda, Edwards, Storkey, and Klimov]{burda2018rnd}
Yuri Burda, Harrison Edwards, Amos Storkey, and Oleg Klimov.
\newblock Exploration by random network distillation.
\newblock In \emph{International Conference on Learning Representations}, 2018.

\bibitem[Burda et~al.(2019)Burda, Edwards, Pathak, Storkey, Darrell, and
  Efros]{burda2018curiosity}
Yuri Burda, Harri Edwards, Deepak Pathak, Amos Storkey, Trevor Darrell, and
  Alexei~A. Efros.
\newblock Large-scale study of curiosity-driven learning.
\newblock In \emph{International Conference on Learning Representations}, 2019.

\bibitem[Iqbal and Sha(2019)]{iqbal2019coordinated}
Shariq Iqbal and Fei Sha.
\newblock Coordinated exploration via intrinsic rewards for multi-agent
  reinforcement learning.
\newblock \emph{arXiv preprint arXiv:1905.12127}, 2019.

\bibitem[Jaksch et~al.(2010)Jaksch, Ortner, and Auer]{jaksch2010ucrl}
Thomas Jaksch, Ronald Ortner, and Peter Auer.
\newblock Near-optimal regret bounds for reinforcement learning.
\newblock \emph{Journal of Machine Learning Research}, 11:\penalty0 1563--1600,
  2010.

\bibitem[Jin et~al.(2018)Jin, Allen-Zhu, Bubeck, and Jordan]{jin2018qlearning}
Chi Jin, Zeyuan Allen-Zhu, Sebastien Bubeck, and Michael~I Jordan.
\newblock Is q-learning provably efficient?
\newblock In \emph{Advances in Neural Information Processing Systems}, pages
  4863--4873, 2018.

\bibitem[Johanson et~al.(2011)Johanson, Waugh, Bowling, and
  Zinkevich]{johanson2011nashconv}
Michael Johanson, Kevin Waugh, Michael Bowling, and Martin Zinkevich.
\newblock Accelerating best response calculation in large extensive games.
\newblock In \emph{IJCAI}, volume~11, pages 258--265, 2011.

\bibitem[Lanctot et~al.(2017)Lanctot, Zambaldi, Gruslys, Lazaridou, Tuyls,
  P{\'e}rolat, Silver, and Graepel]{lanctot2017psro}
Marc Lanctot, Vinicius Zambaldi, Audrunas Gruslys, Angeliki Lazaridou, Karl
  Tuyls, Julien P{\'e}rolat, David Silver, and Thore Graepel.
\newblock A unified game-theoretic approach to multiagent reinforcement
  learning.
\newblock In \emph{Advances in neural information processing systems}, pages
  4190--4203, 2017.

\bibitem[Littman(1994)]{littman1994markov}
Michael~L Littman.
\newblock Markov games as a framework for multi-agent reinforcement learning.
\newblock In \emph{Machine learning proceedings 1994}, pages 157--163.
  Elsevier, 1994.

\bibitem[Osband et~al.(2019)Osband, Van~Roy, Russo, and
  Wen]{osband2019exploration}
Ian Osband, Benjamin Van~Roy, Daniel~J Russo, and Zheng Wen.
\newblock Deep exploration via randomized value functions.
\newblock \emph{Journal of Machine Learning Research}, 20\penalty0
  (124):\penalty0 1--62, 2019.

\bibitem[Osband et~al.(2020)Osband, Doron, Hessel, Aslanides, Sezener, Saraiva,
  McKinney, Lattimore, Szepesvari, Singh, Roy, Sutton, Silver, and
  Hasselt]{osband2019bsuite}
Ian Osband, Yotam Doron, Matteo Hessel, John Aslanides, Eren Sezener, Andre
  Saraiva, Katrina McKinney, Tor Lattimore, Csaba Szepesvari, Satinder Singh,
  Benjamin~Van Roy, Richard Sutton, David Silver, and Hado~Van Hasselt.
\newblock Behaviour suite for reinforcement learning.
\newblock In \emph{International Conference on Learning Representations}, 2020.

\bibitem[Pathak et~al.(2017)Pathak, Agrawal, Efros, and Darrell]{pathak2017icm}
Deepak Pathak, Pulkit Agrawal, Alexei~A Efros, and Trevor Darrell.
\newblock Curiosity-driven exploration by self-supervised prediction.
\newblock In \emph{Proceedings of the IEEE Conference on Computer Vision and
  Pattern Recognition Workshops}, pages 16--17, 2017.

\bibitem[Pearl(1980)]{pearl1980alpha_beta}
Judea Pearl.
\newblock Asymptotic properties of minimax trees and game-searching procedures.
\newblock \emph{Artificial Intelligence}, 14\penalty0 (2):\penalty0 113--138,
  1980.

\bibitem[Strehl and Littman(2008)]{strehl2008mbie}
Alexander~L Strehl and Michael~L Littman.
\newblock An analysis of model-based interval estimation for markov decision
  processes.
\newblock \emph{Journal of Computer and System Sciences}, 74\penalty0
  (8):\penalty0 1309--1331, 2008.

\bibitem[Tan(1993)]{tan1993multi}
Ming Tan.
\newblock Multi-agent reinforcement learning: Independent vs. cooperative
  agents.
\newblock In \emph{Proceedings of the tenth international conference on machine
  learning}, pages 330--337, 1993.

\bibitem[Tesauro(1994)]{tesauro1994td}
Gerald Tesauro.
\newblock Td-gammon, a self-teaching backgammon program, achieves master-level
  play.
\newblock \emph{Neural computation}, 6\penalty0 (2):\penalty0 215--219, 1994.

\bibitem[Vinyals et~al.(2019)Vinyals, Babuschkin, Czarnecki, Mathieu, Dudzik,
  Chung, Choi, Powell, Ewalds, Georgiev, et~al.]{vinyals2019alphastar}
Oriol Vinyals, Igor Babuschkin, Wojciech~M Czarnecki, Micha{\"e}l Mathieu,
  Andrew Dudzik, Junyoung Chung, David~H Choi, Richard Powell, Timo Ewalds,
  Petko Georgiev, et~al.
\newblock Grandmaster level in starcraft ii using multi-agent reinforcement
  learning.
\newblock \emph{Nature}, 575\penalty0 (7782):\penalty0 350--354, 2019.

\bibitem[Weissman et~al.(2003)Weissman, Ordentlich, Seroussi, Verdu, and
  Weinberger]{weissman2003error}
Tsachy Weissman, Erik Ordentlich, Gadiel Seroussi, Sergio Verdu, and Marcelo~J
  Weinberger.
\newblock Inequalities for the l1 deviation of the empirical distribution.
\newblock \emph{Hewlett-Packard Labs, Tech. Rep}, 2003.

\end{thebibliography}
